\pdfoutput=1

\documentclass[11pt]{article}

\usepackage[final]{acl}
\usepackage{amsmath}
\usepackage{float}
\usepackage{fullpage}
\usepackage{times}
\usepackage{booktabs}
\usepackage{fancyhdr,graphicx,amsmath,amssymb}
\usepackage{algorithm}
\usepackage{multirow}
\usepackage{algorithmic}
\usepackage{bm}
\usepackage{times}
\usepackage{latexsym}

\usepackage[utf8]{inputenc}

\usepackage{microtype}

\usepackage{inconsolata}

\usepackage{graphicx}

\usepackage{amsthm}
\theoremstyle{plain}
\newtheorem{theorem}{Theorem}

\newtheorem{lemma}{Lemma}

\theoremstyle{definition}

\theoremstyle{remark}

\title{AdaZeta: Adaptive Zeroth-Order Tensor-Train Adaption for Memory-Efficient Large Language Models Fine-Tuning}
 
\author{Yifan Yang$^1$\:\:Kai Zhen$^2$\:\:Ershad Banijamali$^2$\:\:Athanasios Mouchtaris$^2$\:\:Zheng Zhang$^1$ \\
  $^1$University of California, Santa Barbara \\
  $^2$Amazon AGI \\
  yifanyang@cs.ucsb.edu\quad \{kaizhen, ebanijam, mouchta\}@amazon.com\\ zhengzhang@ece.ucsb.edu}
\begin{document}
\maketitle
\begin{abstract}
Fine-tuning large language models (LLMs) has achieved remarkable performance across various natural language processing tasks, yet it demands more and more memory as model sizes keep growing. To address this issue, the recently proposed Memory-efficient Zeroth-order (MeZO) methods attempt to fine-tune LLMs using only forward passes, thereby avoiding the need for a backpropagation graph. However, significant performance drops and a high risk of divergence have limited their widespread adoption. In this paper, we propose the Adaptive Zeroth-order Tensor-Train Adaption (AdaZeta) framework, specifically designed to improve the performance and convergence of the ZO methods. To enhance dimension-dependent ZO estimation accuracy, we introduce a fast-forward, low-parameter tensorized adapter. To tackle the frequently observed divergence issue in large-scale ZO fine-tuning tasks, we propose an adaptive query number schedule that guarantees convergence. Detailed theoretical analysis and extensive experimental results on Roberta-Large and Llama-2-7B models substantiate the efficacy of our AdaZeta framework in terms of accuracy, memory efficiency, and convergence speed.\footnote[1]{Code available on GitHub \url{https://github.com/yifanycc/AdaZeta}.}\footnote[2]{Accepted by EMNLP 2024}
\end{abstract}
\begin{figure*}[t]
    \centering
    \includegraphics[width=1\linewidth]{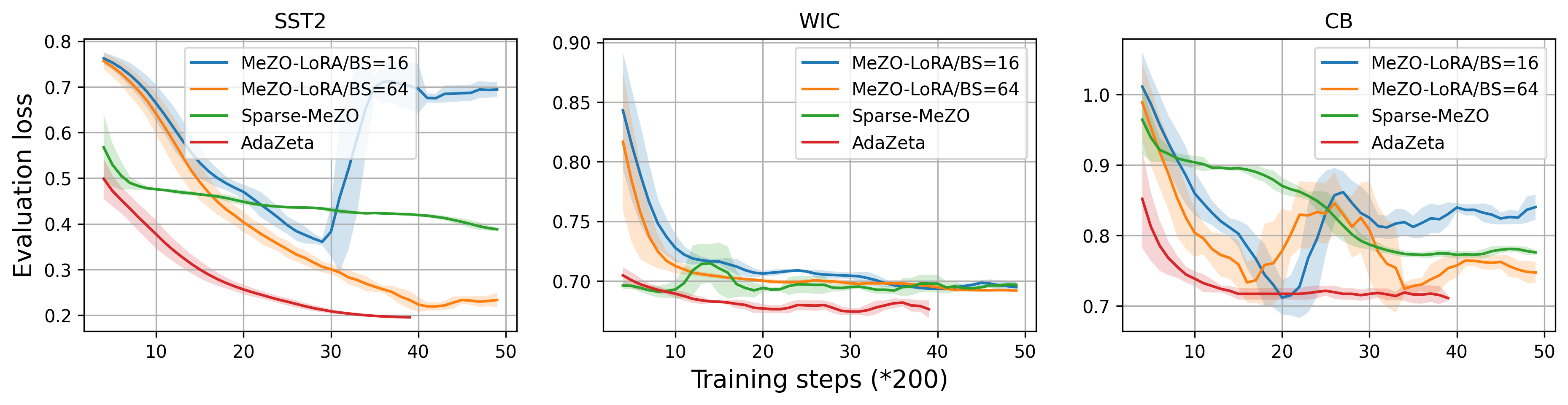}
    \caption{The evaluation loss curves for the SST-2, WiC, and CB tasks using the Llama-2-7B model. The proposed AdaZeta method converges faster and effectively addresses the divergence problem using a much smaller batch size (BS). Both MeZO-LoRA and AdaZeta use a learning rate of 1e-4, while Sparse-MeZO utilizes a 1e-6 learning rate.}
    \label{fig:main}
    \vspace{-10pt}
\end{figure*}
\section{Introduction}
Fine-tuning large language models (LLMs) has demonstrated outstanding performance in addressing numerous natural language processing applications, such as natural language understanding \cite{kenton2019bert}, question-answering \cite{xu2024can, cheng2023ghostt5}, and summarization \cite{zhang2024benchmarking}. However, as the size of LLMs increases, the training process consumes progressively more GPU memory. In recent years, approaches such as quantization \cite{tian2023bebert, dettmers2024qlora} and parameter-efficient fine-tuning (PEFT) \cite{hu2021lora} have been proposed to reduce memory costs during training by storing data with lower bit-depth or updating only a portion of the parameters. Despite these strategies effectively reducing memory costs, overall memory usage remains high due to the continuous reliance on a backpropagation graph.\\~\\
To further reduce the memory overhead, \cite{malladi2023fine} proposed the Memory-efficient Zeroth-order (MeZO) method for LLM fine-tuning, which shows over 8$\times$ memory reduction compared with the first-order (FO) fine-tuning methods like SGD \cite{amari1993backpropagation} and AdamW \cite{loshchilov2018decoupled}. Unlike FO methods, which calculate gradients via backpropagation, the MeZO method estimates gradients based on the difference between loss values obtained from two forward passes, thereby eliminating the need for a backpropagation graph. However, two main challenges persist in the zeroth-order (ZO) fine-tuning of LLMs: 1) a significant performance gap between FO and ZO approaches, and 2) increased risk of divergence, particularly in the ZO fine-tuning of large-scale LLMs, as observed in recent studies \cite{gautam2024variance}.\\~\\
To improve the performance, various FO optimization techniques have been adapted for ZO fine-tuning scenarios, like the ZO-AdaMU method \cite{jiang2024zo}. However, these approaches fail to accommodate the specific needs of ZO methods, and add significant memory overhead from the optimizer state. Given the dimensionality-related nature of ZO convergence rates, \cite{liu2024sparse} propose the Sparse-MeZO method that generates pruning masks based on the value of the weight elements. Nevertheless, the Sparse-MeZO method yields inconsistent performance across various tasks and hyperparameter configurations. In contrast to this approach, we consider using the PEFT method to reduce the number of trainable parameters. Although the ZO PEFT method like MeZO-LoRA has been considered in \cite{malladi2023fine}, the improvements are limited as the LoRA adapter fails to offer high representational ability with an ultra-low rank. To solve this problem, we involve tensorized adapters, which offer high performance with even lower trainable parameters than LoRA adapters.\\~\\
To address the variance-related divergence issue in large-scale ZO fine-tuning, previous studies \cite{malladi2023fine, jiang2024zo} have primarily focused on adjusting the batch size, as increasing the batch size can reduce the noise in ZO gradient estimation. However, these approaches introduce significant runtime overhead and fail to improve performance significantly. To further reduce variance, \cite{gautam2024variance} introduced the MeZO-SVRG method, adapting the first-order SVRG technique to the ZO context. Despite its success, MeZO-SVRG suffers from a slow and memory-inefficient fine-tuning process due to the additional parameter copies and computation process that even doubles the memory cost of the MeZO methods. In contrast to these works, we consider reducing the ZO gradient variance with a sublinearly increasing query\footnote[2]{A query refers to request the gradient of the loss function for one time in this paper \cite{bubeck2015convex}[Sec. 4.1.4].} schedule that achieves not only better accuracy but also faster convergence in terms of both steps and time.\\~\\
This paper explores task-specific PEFT training for ZO fine-tuning scenarios. We introduce the Adaptive Zeroth-order Tensor-Train Adaption (AdaZeta) framework, which incorporates fast-forward tensorized adapters and an adaptive query schedule. This combination can significantly enhance the accuracy and convergence of ZO fine-tuning, as demonstrated in Fig. \ref{fig:main}. Our contributions are summarized as follows:
\begin{itemize}
\item We introduce the AdaZeta framework, outperforming other ZO fine-tuning methods like MeZO, MeZO-LoRA, and Sparse-MeZO across different tasks with faster convergence.\vspace{-8.1pt}
\item We develop an adaptive query number schedule that sub-linearly increases the number of queries to address the persistent divergence issue in ZO fine-tuning.
\vspace{-8.1pt}
\item We provide both theoretical and experimental results to demonstrate the training efficiency and performance of our method.
\end{itemize}
\section{Background}
\subsection{Parameter-Efficient Fine-tuning}
In recent years, various works related to PEFT methods have been proposed. Beyond the most widely used methods like Adapters \cite{houlsby2019parameter} and LoRA \cite{hu2021lora}, there are also methods exploring ultra-low trainable parameter solutions \cite{zaken2022bitfit, li2021prefix, liu2022few}. In \cite{malladi2023fine}, researchers try to employ the LoRA and prefix-tuning \cite{li2021prefix} methods during the ZO fine-tuning. However, the improvement is limited and the detailed analysis of ZO PEFT tuning is not discussed. \\~\\
In this paper, we explore tensorized adapters, an ultra-low-parameter PEFT method that compresses the weight matrices of adapter layers using Tensor-Train (TT) decomposition. This approach is examined in \cite{yang2024loretta}, where it demonstrates strong performance in FO fine-tuning tasks. However, the contraction process of TT format \cite{oseledets2011tensor, novikov2015tensorizing} involving a sequence of small tensor factors slows down the forward pass, making it less suitable for ZO methods that require two forward passes per step. To solve this problem, we propose parallel contraction methods to improve the inference speed of tensorized adapter methods.
\begin{figure}[t]
\centering
\includegraphics[width=0.51\textwidth]{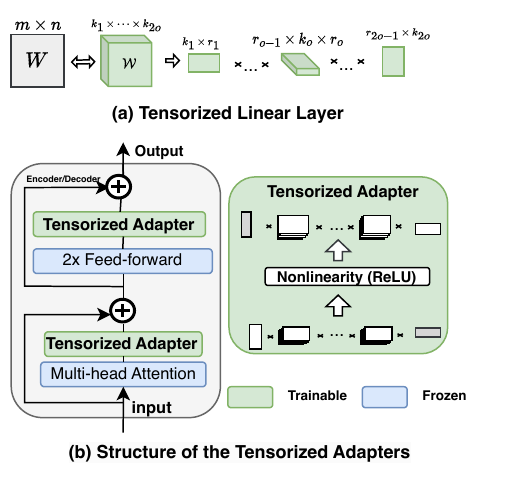}
\vspace{-30pt}
\caption{Illustration for tensorized linear layer and tensorized adapters.}
\label{fig:adapters}
\end{figure} 
\subsection{Tensorized Adapters}
As shown in Fig. \ref{fig:adapters} (a), the tensorized adapters, which are built upon tensorized linear layers, are lightweight components injected during the fine-tuning process to reduce the number of trainable parameters. The weight in tensorized linear layers is represented in the TT format. Compared with a standard weight matrix $\bm{W}\in\mathbb{R}^{m\times n}$ in a typical linear layer, the TT format represents its reshaped $2o$-way tensor $\mathcal{W}\in\mathbb{R}^{k_1\times\cdots\times k_{2o}}$ as a sequence of tensor factors $[\mathcal{G}_{1}, \cdots, \mathcal{G}_{o}, \mathcal{G}_{o+1}, \cdots \mathcal{G}_{2o}]$ \cite{oseledets2011tensor}, where each tensor factor $\mathcal{G}_i \in \mathbb{R}^{r_{i-1} \times k_i \times r_i}$ has rank $r_{i-1}$ and $r_i$. The dimensions $k_i$ are constrainted such that $\Pi_{i=1}^o k_i = m$ and $\Pi_{j=o+1}^{2o} k_j = n$. During the forward pass, the sequence of tensor factors is contracted and reshaped back into the shape of a weight matrix as 
 \begin{align}\label{eq:contraction}
     \bm{W} = \text{Reshape}(\mathcal{G}_1\times \cdots\times \mathcal{G}_{2o}).
 \end{align}
Note that in this paper, the tensor rank is held constant, with the exception of the first and last ranks, which are set $r_0=r_{2o}=1$. Also, the weights in tensorized layers are initialized, stored, and updated in TT-format instead of the matrix form in a traditional linear layer. \\~\\
The structure of tensorized adapters is shown in Fig. \ref{fig:adapters} (b). Each tensorized adapter contains two tensorized layers and a non-linear layer in between. For each encoder/decoder block, the tensorized adapters are attached after the attention and feed-forward layer. Different from \cite{yang2024loretta} that makes both tensorized adapters and layer norm trainable, we freeze the layer norm during the ZO fine-tuning, as noisy gradient estimation of the scaling factor in layer normalization can seriously degrade model performance. The tensorized adapters reduce trainable parameters by over $80\times$, making them a better fit for ZO fine-tuning.
\section{Methods}
In this section, we first introduce some basic knowledge of the ZO gradient estimator. Then, we present our AdaZeta method, a powerful framework designed to improve the performance of ZO LLM fine-tuning with two main components: 1) the fast-forward tensorized adapters, and 2) an adaptive query number schedule. Finally, we provide a theoretical analysis of the convergence rate of the AdaZeta method, demonstrating the improved convergence rate theoretically.
\subsection{Zeroth-order Estimation}
Traditional ZO estimation has been widely studied in both convex and non-convex optimization setups \cite{ghadimi2013stochastic, malladi2023fine, chen2019zo}. In our problem, considering a supervised dataset $\mathcal{D}$, mini-batch $\mathcal{B}$ with the size of $D$ and $B$ respectively, we set the loss function for our fine-tuning problem to be $\ell(\bm{w}; \mathcal{B})$, where the trainable parameter in the tensorized adapters $\bm{w}\in \mathbb{R}^d$ has a size of $d$. Then, the Randomized Zeroth-order Gradient Estimation (RGE) at training step $k$ is given as:
\begin{equation*}
\nabla\hat{\ell}(\bm{w}_k) = \sum_{q=1}^{Q_k}\frac{\ell_\mathcal{B}(\bm{w}_k + \epsilon \bm{z}_q) - \ell_\mathcal{B}(\bm{w}_k - \epsilon \bm{z}_q)}{2\epsilon}\bm{z}_q
\end{equation*}
where $Q_k$ is the query number at the training step $k$, $\bm{z}_q \sim \mathcal{N}(0, \mathbb{I}_d)$ is the vector-wise random perturbation for each query $q$, and $\epsilon$ is a scaling factor for the perturbation.\\~\\
Unlike FO fine-tuning, which relies on backpropagation, RGE requires only two forward passes with perturbations added to the weights of tensorized adapters, eliminating the need for a backpropagation graph. Additionally, by sublinearly increasing the number of queries at the beginning of each epoch, we effectively reduce the variance of the ZO gradient estimation by involving distinct perturbations $\bm{z}_q$ at each time of query. Details of the setup will be discussed in the following section.
\begin{algorithm}[t]
\textbf{Input:} Parameters $\bm{w}$, loss function $\ell(\cdot)$, random seed $s_q$, scaling factor $\epsilon$, Query-realted constant $\alpha, \beta$, maximum query $Q_{max}$, learning rate $\eta$.
\begin{algorithmic}[1]
 \FOR{$k=1, \cdots, K$}
 \STATE Calculating query number at epoch $e_k$ start:\vspace{-5pt}$$Q_k:=\min(\alpha e_k^{\beta},Q_{max})$$\vspace{-20pt}
\FOR{$q=1, \cdots, Q_k$}
  \STATE$\bm{w}\leftarrow\bm{w} + \epsilon \bm{z}_q,\quad \bm{z}_q\sim \mathcal{N}(0,\mathbb{I}_d, s_q)$
  \STATE$\ell_+^q\leftarrow \ell(\bm{w}, \mathcal{B})$
  \STATE$\bm{w}\leftarrow\bm{w} -2 \epsilon \bm{z}_q, \quad\bm{z}_q\sim \mathcal{N}(0,\mathbb{I}_d, s_q)$
  \STATE$\ell_-^q\leftarrow \ell(\bm{w}, \mathcal{B})$
  \STATE $\bm{w}\leftarrow\bm{w} + \epsilon \bm{z}_q,\quad \bm{z}_q\sim \mathcal{N}(0,\mathbb{I}_d, s_q)$
   \STATE Reset random seed $s_q$ for generating $\bm{z}_q$ 
    \ENDFOR
   \STATE   $\nabla_{\bm{w}} \hat{\ell}(\bm{w})=\frac{1}{Q_k} \sum_{q=1}^{Q_k}\left[\frac{\ell_+^q - \ell_-^q}{2\epsilon} \bm{z}_q\right]$
   \STATE $\bm{w}\leftarrow\bm{w} - \eta *\nabla_{\bm{w}} \hat{\ell}(\bm{w})$
  \ENDFOR
 	\end{algorithmic}  
\caption{AdaZeta Algorithm}\label{alg:ttzo}
\end{algorithm}
\subsection{The AdaZeta Framework}
Previous ZO fine-tuning methods, such as MeZO, typically estimate the gradient for a large number of trainable parameters simultaneously using RGE. This approach results in high variance due to the dimension-related nature of the RGE method. Although techniques like LoRA and prefix tuning have been considered, few works consider the tasks-specific PEFT adapters for the ZO LLMs fine-tuning. Additionally, as shown in Fig. \ref{fig:main}, we have observed an increased risk of divergence when using the MeZO-LoRA method during fine-tuning. To address these issues, we propose our AdaZeta framework to improve performance and solve the instability problem of the vanilla MeZO method. Our framework includes the following components:\\~\\
\textbf{Fast Forward Tensorized Adapters.} The Parameter-efficient issue has been widely studied in the FO cases, where people often freeze the pre-trained model parameters and fine-tune the LLMs by adding trainable adapters along with the frozen pretrain weights. Since the ZO estimation accuracy is dimension-dependent, reducing dimensionality can significantly help improve the gradient estimation quality. Thus, we consider injecting the ultra-low parameter tensorized adapters in our AdaZeta framework to reduce the number of trainable parameters while retaining the performance.\\~\\
As we have mentioned, ZO fine-tuning mainly relies on gradient estimation with two forward passes at each step. Thus, the speed of the forward pass is a crucial factor for the overall speed of ZO fine-tuning. Instead of using the sequential contraction method during the forward pass as in previous work, we propose a new parallel contraction method to speed up the forward passes. This method divides the sequence of tensor factors into several groups to enable parallel processing and avoid the presence of high-dimensional tensors. Taking a bipartite case as an example, the contraction process in eq. (\ref{eq:contraction}) is replaced by:
\begin{align*}
    \bm{W} = \text{R}(\prod_{i=1}^{o} \mathcal{G}_i \prod_{j=o+1}^{2o} \mathcal{G}_j),
\end{align*}
where \(\mathcal{G}_i\) represents the \(i\)-th tensor factor, \(\text{R}(\cdot)\) represents the reshape operation. For larger models, the tensor factors can be organized into tripartite or quadripartite structures to accelerate the inference speed of the tensorized methods. \\~\\
\textbf{Adaptive Query Adjustment for ZO estimation.} 
As previously noted, the training process for existing ZO methods often exhibits instability, particularly with large-size models where divergence issues frequently occur. Previous studies \cite{chen2019zo, jiang2024zo} have explored using a fixed multiple queries scheme to improve the estimation accuracy in the optimization community. However, utilizing a fixed number of queries may significantly hinder the training efficiency of large-scale ZO fine-tuning tasks, as naively increasing the number of perturbations greatly escalates training durations. To solve this problem, we consider a simple but effective sublinear increasing query number adjustment schedule, where the number of queries is updated at the beginning of each epoch $e_k$. By expressing the epoch in terms of the global training steps as $e_k=\lfloor k / \lceil \frac{D}{B}\rceil \rfloor$, we have:
\begin{equation}\label{eq:query}
 Q_k:=\min(\alpha e_k^{\beta},Q_{max})   
\end{equation} 
with a fixed scaling factor $\alpha\in(0,1)$, a sublinear increasing factor $\beta\in(0,1)$ and a max query threshold $Q_{max}$. Then, the query number is fixed for all training steps within each epoch. This adjustment solves all divergence problems we observed with theoretical guarantee and performs even faster than the traditional way to solve the divergence problem for ZO LLMs fine-tuning by increasing the batch size.\\~\\
The corresponding optimization algorithm used in the AdaZeta framework is shown in Alg. \ref{alg:ttzo}. We adjust the query number at the beginning of each epoch. Different from the MeZO algorithm, we obtain the gradient used for the model update by taking the average over multiple query results. Note that we fix the query number to be 1 when fine-tuning medium-size models like Roberta-Large since the noise of ZO estimation is relatively low when the number of trainable parameters is small. Later, we will show that a sublinear increasing query number benefits the convergence of the problem when the model size is large, both theoretically and experimentally.
\begin{table*}[th]
\centering
\caption{Comparative analysis of various ZO fine-tuning methods on the Roberta-Large models.}
\label{tab:bert}
\resizebox{0.6\textwidth}{!}{%
\begin{tabular}{cccccccc}
\hline
\multicolumn{1}{c|}{Methods}    & RTE      & SST-2         & SST-5         & QNLI          & MNLI          & SNLI          & MR                   \\ \hline
\multicolumn{1}{c|}{FT}          & 66.4     & 91.9          & 47.5          & 63.4          & 70.0            & 77.5          & 88.2                    \\
\multicolumn{1}{c|}{Zero-Shot}    & 51.4    & 79.0            & 35.5          & 50.9          & 48.8          & 50.2          & 80.2                \\
\multicolumn{1}{c|}{LP}        & 59.4      & 76.0            & 40.3          & 57.6          & 56.5          & 66.0            & 86.6                  \\ \hline
\multicolumn{8}{c}{BS=16}                                                                                                                           \\ \hline
\multicolumn{1}{c|}{MeZO}     & 52.7      & 90.5          & 31.1          & 59.9          & \textbf{60.5} & 63.5          & 85.5                 \\
\multicolumn{1}{c|}{MeZO-LoRA}    & 52.7     & 84.2          & 44.8          & 60.3          & 58.5          & 65.6          & 85.7                  \\
\multicolumn{1}{c|}{AdaZeta}     & \textbf{66.8}          & \textbf{91.4} & \textbf{48.3} & \textbf{61.3} & 58.1          & \textbf{69.1} & \textbf{87.0}           \\ \hline
\multicolumn{8}{c}{BS=64}                                                           \\ \hline
\multicolumn{1}{c|}{MeZO}     & 64.0        & 90.5          & 45.5          & 60.5          & 58.7          & 68.5          & 85.0                     \\
\multicolumn{1}{c|}{MeZO-LoRA}  & 63.9   & 91.3          & 43.0            & 59.0            & 64.0            & \textbf{69.7} & \textbf{87.4}          \\
\multicolumn{1}{c|}{\textbf{AdaZeta}} & \textbf{64.3} & \textbf{91.5} & \textbf{49.6} & \textbf{60.7} & \textbf{68.1} & 68.7          & 86.5          \\ \hline
\end{tabular}
}
\end{table*}
\subsection{Theoretical Analysis}\label{sec:the}
In this subsection, we give the theoretical analysis for the AdaZeta framework. Our theoretical analysis highlights why the tensorized adapter and adaptive query schedule can significantly help to improve the ZO convergence rate. Unlike the theoretical analysis in the MeZO paper, which focuses on the "effective rank" for the Hessian of loss, we focus on the dimension of the optimized models $d$ (number of trainable parameters) instead. As the trainable parameters with PEFT adapters are much smaller than the model size, the theoretical analysis based on the exact dimension of the optimization problem can better help us explore the behavior of different PEFT methods. \\~\\
To align our analysis with LLM fine-tuning, we consider a non-convex optimization setup and study the convergence behavior regarding the training steps $k$. It is important to note that the ZO estimated gradient $\nabla\hat{\ell}$ by the RGE, is an unbiased estimation of the true gradient $\nabla\ell$ when $\epsilon\rightarrow 0$, which gives the fact $\mathbb{E}_{\bm{z}}[\nabla\hat{\ell}]=\nabla\ell$ \cite{nesterov2017random}. First, we list the following assumptions for our analysis:\\~\\
\textbf{A1:} The loss function $\ell$ has an L-Lipschitz continuous gradient, where for $L>0$ we have:
\vspace{-8.1pt}
\begin{equation*}
    \|\nabla\ell(\bm{w}_i) - \nabla\ell(\bm{w}_j)\|\leq L\|\bm{w}_i - \bm{w}_j\|, \forall \bm{w}_i, \bm{w}_j
\end{equation*}\vspace{-30.1pt}\\~\\
\textbf{A2:} At each step $k$, the gradient of loss function $\ell$ is upper bounded as $\|\nabla \ell\| \le \delta, \forall k$.\\~\\
Then, we offer the global convergence rate for our AdaZeta algorithm:
\begin{theorem}\label{the:convergence}
    Under A1 and A2, randomly pick $\bm{w}_T$ from history with probability $P(T=k)=\frac{1}{K}$, the convergence of the AdaZeta algorithm can be bounded by:
    \begin{align*}
   \mathbb{E}[\|\nabla\ell(\bm{w}_T)\|^2] \leq\mathcal{O}(\frac{R + \epsilon^2 L +C(d,\epsilon) \sum_{k}\frac{1}{Q_k}}{K\epsilon}),
\end{align*}
where $R$ is defined by the distance between the start point and the optimal solution $\ell(\bm{w}_1) - \ell^*$, the ZO perturbation scaling factor is represented as $\epsilon$, and $C(d,\epsilon)$ is a constant related to the model parameter size $d$, which is defined at the end of the proof in Appendix \ref{app:proof}.
\end{theorem}
\begin{proof}
    Details can be found in Appendix \ref{app:proof}.
\end{proof}
\noindent According to Theorem \ref{the:convergence}, we can observe that the bound is related to the query schedule. For convenience, take a simplified case with $\alpha=\beta=0.5$ and ignore the minimum in eq. (\ref{eq:query}), we have $Q_k = \frac{1}{2}\sqrt{\lfloor k / \lceil \frac{D}{B}\rceil \rfloor}$, gives $\sum_{k=1}^K\frac{1}{Q_k}\leq2\left\lceil\frac{D}{B}\right\rceil \sqrt{\left\lfloor \frac{K}{\lceil \frac{D}{B}\rceil} \right\rfloor} $,
which guarantees the true gradient approaches zero when $K\rightarrow\infty$. In contrast, using a small constant such as \( Q = 1 \) results in an upper bound of $\mathcal{O}(C(d, \epsilon)/K\epsilon)$, which becomes challenging to minimize due to the term $C(d, \epsilon)$ is directly proportional to the model size $d$. Additionally, we observe that the convergence rate is significantly influenced by the model dimension $d$. Consequently, in this paper, we also try to reduce the number of trainable parameters with the tensorized adapters.
\begin{table*}[t]
\centering
\caption{Comparative analysis of various ZO fine-tuning methods on the Llama-2-7B model.}
\label{tab:llama}
\resizebox{0.9\textwidth}{!}{%
\begin{tabular}{c|cccccccccc}
\hline
Methods               & RTE         & CB          & BoolQ         & WSC           & WIC   & SST2       & MultiRC       & COPA        & ReCoRD        & SQuAD  \\ \hline
FT                  & 61.7       & 66.1        & 84.6          & 63.4         & 65.9    & 94.0   & 45.4          & 86.0          & 81.1          & 90.7  \\
LoRA               & 85.5       & 67.8       & 84.8          & 62.5          & 73.9     & 94.8   & 85.0            & 81.0          & 79.4          & 90.5  \\ \hline
Zero-Shot           & 49.5        & 32.1        & 65.1          & 36.5          & 50.6    & 79.7     & 55.8          & 59.7        & \textbf{80.9} & 54.7    \\
ICL                  & 54.5        & 58.9        & 67.4          & \textbf{65.4} & 52.7     & 81.2    & 58.7          & 84.4        & 80.1          & 67.1 \\
MeZO                & 54.6        & 73.0          & 68.6          & 52.8          & 57.8      & 85.8  & 62.6          & 86.0          & 70.8          & 72.5    \\
MeZO-LoRA             & 59.6          & 74.0          & 71.6          & 53.0          & 55.2   & 86.8     & 67.2          & 89.0 & 72.0            & 80.0     \\
Sparse-MeZO    &	58.6	&76.0	&67.8	&53.0	&56.8&85.2	&61.2&	86.0	&70.6	&64.4 \\
\textbf{AdaZeta} & \textbf{74.0} & \textbf{75.0} & \textbf{79.4} & 52.2          & \textbf{58.0} &\textbf{91.0}& \textbf{68.2} & \textbf{94.0}          & 71.2          & \textbf{80.0  }        \\ \hline
\end{tabular}
}
\end{table*}
\section{Experiments}
In this section, we conduct comprehensive experiments to evaluate the performance of our proposed AdaZeta framework across several LLMs with different scales on a variety of natural language understanding and generation tasks \cite{socher2013recursive, williams2017broad, rajpurkar2016squad}.
We demonstrate that our methods surpass a comprehensive array of memory-efficient baselines, including inference-only methods such as Zero-shot \cite{brown2020language}, In-Context Learning (ICL), and Linear Probing (LP) \cite{kumar2021fine}, as well as ZO fine-tuning methods like MeZO, MeZO-LoRA \cite{malladi2023fine}, and Sparse-MeZO \cite{liu2024sparse}. Also, the first-order fine-tuning (FT) baseline is also provided as a reference. \\~\\
Initially, we present experimental evidence using Roberta-Large models \cite{liu2019roberta}, illustrating that the integration of tensorized adapters can significantly enhance the efficiency of ZO fine-tuning by reducing the number of trainable parameters. Subsequently, we enabled our proposed adaptive query schedule method to show the effectiveness of the AdaZeta framework on large-scale Llama-2-7B models \cite{touvron2023LLaMA}, which not only enhances performance but also ensures robust convergence. All experiments are conducted on NVIDIA Tesla A100-40GB GPUs, with further details about the experimental setup available in Appendix \ref{app:exp}.
\subsection{Medium-size Roberta-Large Models}
We initially evaluated the effectiveness of using tensorized adapters on RoBERTa-large models across various tasks, including single-sentence tasks like SST-2 and SST-5, natural language inference tasks such as QNLI, MNLI, SNLI, RTE, and the sentiment analysis dataset Movie Reviews (MR). The results are summarized in Table \ref{tab:bert}. Experiments were conducted under a 16-shot setup, with 16 data samples in each class of the datasets. We monitored the best test accuracy every 500 steps, using a test pool of 1,000 data samples. Note that, similar to previous ZO fine-tuning studies, we \textbf{fixed the number of queries to 1 in this subsection}. This decision is based on the observation that gradient noise is relatively small in medium-sized Bert-based models. The following conclusions have been reached:\\~\\
\textbf{AdaZeta Shows Higher Accuracy than Other ZO Fine-Tuning Methods.} According to our observations in Table \ref{tab:bert}, AdaZeta outperforms other ZO fine-tuning approaches in terms of evaluation accuracy. Compared with MeZO-LoRA, which also involves PEFT adapters, AdaZeta outperforms in 5 out of 7 tests under both 16 and 64 batch size (BS) settings. This advantage shows the effectiveness of improving ZO estimation accuracy by further reducing the number of trainable parameters with the tensorized adapter. This is supported by the dimension-related convergence rate proved in Section \ref{sec:the}.\\~\\
\textbf{AdaZeta Demonstrates Improved Convergence.} Compared to the MeZO-LoRA method, the AdaZeta method exhibits superior convergence when the batch size is 16. Given our 16-shot training setup, it is reasonable to expect that the 16 batch size scenario would outperform the 64 batch size scenario if the fine-tuning process converges effectively. However, a performance decline is observed with the MeZO-LoRA method, indicating that it is adversely affected by ZO gradient noise. Comparatively, the AdaZeta method achieves consistent results across both setups by reducing such noise with less trainable parameters, effectively showcasing its ability to aid in convergence.
\subsection{Large-scale Llama-2 Models}
In the previous section, we demonstrated how utilizing the tensorized adapter method enhances ZO fine-tuning performance by reducing gradient noise through a decrease in trainable parameters. In this section, we assess the effectiveness of the AdaZeta framework with the large-scale Llama-2-7B model. Differing from the experiments on the Roberta-Large models, we enabled the adaptive query schedule method proposed in our AdaZeta framework to mitigate the commonly observed divergence issues in large-scale ZO fine-tuning.\\~\\
To highlight the challenge of our experiments, we adopt a low-data resource approach using datasets from SuperGLUE \cite{wang2019superglue} and generative tasks such as SQuAD \cite{rajpurkar2016squad} and DROP \cite{dua2019drop}. Our experimental protocol follows the prompted-based fine-tuning strategy outlined in the MeZO paper \cite{malladi2023fine}. The quantitative results are summarized in Table \ref{tab:llama} and the training curves have been shown in Fig. \ref{fig:main}. Note that it is reasonable to observe some large accuracy gap between different methods under different tasks, which has also been observed in previous MeZO and PEFT papers \cite{malladi2023fine, hu2023llm}. The following conclusions are drawn:\\~\\
\textbf{AdaZeta Method Demonstrates Superior Performance Over Traditional ZO Fine-Tuning.} The AdaZeta framework delivers exceptional accuracy results across a variety of tasks, outperforming all ZO baseline methods such as MeZO and MeZO-LoRA in 8 out of 10 tasks. Compared with traditional inference-only methods like ICL and Zero-shot, AdaZeta significantly surpasses them with respect to test accuracy. Moreover, the AdaZeta method even outperforms the FO-AdamW methods over several tasks like RTE, CB, and COPA, which require 8$\times$ more GPU memory.\\~\\
\textbf{AdaZeta Method Effectively Addresses Divergence Issues in ZO Fine-Tuning.} We can observe from the table that the MeZO and MeZO-LoRA methods achieve unsatisfied results in some tasks like SST2, RTE, and BoolQ compared with our proposed method, which is led by the convergence issue. Also, we have shown that the AdaZeta method achieves lower evaluation loss much faster than the MeZO-LoRA and Sparse-MeZO methods across all tasks in Fig.~\ref{fig:main}. 
For example, the MeZO-LoRA method requires nearly 6K steps to achieve a loss of 0.4, whereas the AdaZeta method achieves the same degree of loss minimization in less than 1K steps, which represents a 6$\times$ speed-up with the same 1e-4 learning rate.
Traditional ways to solve such divergence issues through increasing the batch size are hard to follow in the large-scale LLMs fine-tuning tasks. In contrast, the adaptive query schedule in the AdaZeta framework successfully mitigates this issue without increasing the training memory, thereby improving training outcomes. Additionally, we observed that combining LoRA with the adaptive query schedule significantly improves performance in certain tasks. Future work could also explore incorporating the adaptive query schedule into the MeZO-LoRA method to further enhance stability.
\begin{table}[t]
\centering
\caption{Required GPU hours (GPU numbers $\times$ Training hours) to achieve each evaluation loss for different ZO fine-tuning methods on Llama-2-7B model.}
\label{tab:time}
\resizebox{0.45\textwidth}{!}{%
\begin{tabular}{c|cccc}
\hline
Methods          & SST2 & WIC  & CB   & MultiRC      \\ \hline
MeZO-LoRA(BS=64) & 3.0  & 4.8  & 8,6   &   30.0       \\
MeZO-LoRA(BS=16) & 0.6  & 1.1  & 3.1  & 10.8          \\
Sparse-MeZO      & 4.1  & 3.6  & 4.3  & 6.4 \\
AdaZeta     & 1.1  & 1.0 & 0.9 & 12.1          \\ \hline
\end{tabular}
}
\end{table}
\begin{figure}[t]
    \centering
\includegraphics[width=1\linewidth]{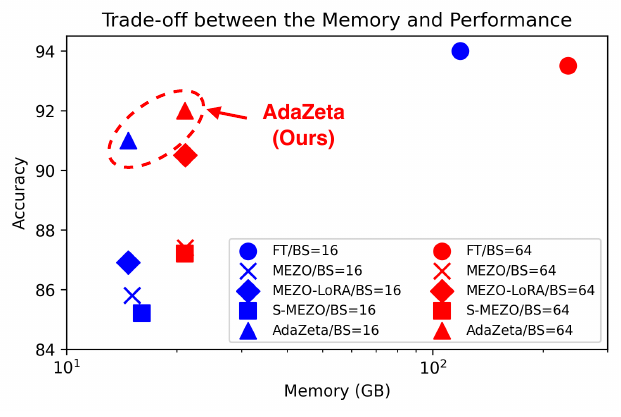}
    \caption{Trade-off between the accuracy and memory cost for different fine-tuning methods. We can observe that the AdaZeta method achieves the best accuracy among the memory-efficient 
methods.}
    \label{fig:memory}
\end{figure}
\subsection{Memory Training Time Efficiency}\label{sec:mem}
In this section, we evaluate the memory and time efficiency of the AdaZeta method. Specifically, we test the peak memory cost of different fine-tuning methods over the Llama-2-7B model and study the trade-off between memory, accuracy, and training time. The result is summarized in Fig. \ref{fig:memory} and further discussion about training memory can be referred to Appendix \ref{app:memory}.\\~\\
According to Fig. \ref{fig:memory} (refer to Appendix \ref{app:memory} for numerical results), the AdaZeta method requires only 14GB of memory to fine-tune the SST2 tasks on the Llama-2-7B model, which achieves over 8$\times$ Memory Reduction Relative to the FT Method. Also, compared with other ZO fine-tuning methods like MeZO, MeZO-LoRA, and Sparse-MeZO, the AdaZeta method utilizes similar or even less memory to achieve variance reduction. Traditional ways to reduce the ZO gradient estimation noise like increasing the batch size, consume significantly more memory than the AdaZeta method as shown in Fig. \ref{fig:memory}.  \\~\\
In Table \ref{tab:time}, we measure the total GPU hours required to achieve a certain threshold of training loss across four tasks (SST2, WIC, CB, MultiRC). For the applicability of the experiments, we established an evaluation loss threshold that all methods could achieve. According to the results, it is evident that the AdaZeta method converges on-par or faster than other ZO fine-tuning methods with even better results than the MeZO-LoRA and Sparse-MeZO methods under the large-batch size case. Note that we \textbf{did not} utilize the gradient accumulation technique for the 64 batch size case, which may significantly increase the training time.
\begin{table*}[t]
\centering
\caption{Compare with first-order LoRA method under low ranks and batch sizes.}
\label{tab:lora}
\resizebox{0.98\textwidth}{!}{%
\begin{tabular}{c|cccccc}
\hline
\textbf{Setup}   & \textbf{LoRA/r=1/BS=1} & \textbf{LoRA/r=1/BS=8} & \textbf{LoRA/r=8/BS=8} & \textbf{AdaZeta/r=8/BS=1} & \textbf{MeZO-LoRA/r=8/BS=16} & \textbf{AdaZeta/r=8/BS=16} \\ \hline
Memory (GB) & 35.60                   & 96.65                  & 96.72                  & 14.05   & 23.02                  & 23.01                      \\ \hline
\end{tabular}
}
\end{table*}
\subsection{Further Comparison with LoRA}
In this section, we further compare our AdaZeta method with the first-order LoRA method in terms of training memory usage across different ranks and batch sizes. The results for the CB task are presented in Table \ref{tab:lora}. We make the following observations under two scenarios:\\~\\
\textbf{Reducing the LoRA Rank:} Reducing the LoRA rank (even down to 1) has minimal impact on training memory in the first-order setting. The reason is that the backpropagation graph—which contains intermediate gradient information—still needs to be retained, spanning almost the entire model in the vanilla LoRA approach.\\~\\
 \textbf{Reducing the Batch Size:} Reducing the batch size is a more effective way to reduce the training memory for both FO and ZO cases. With the existence of a backpropagation graph, it is reasonable to observe a larger reduction of training memory of the FO method than ZO when reducing the number of batch sizes. However, we can observe that even when comparing our method with the LoRA method using a batch size of 1, our method is still 2.5$\times$ more memory-efficient. Additionally, even comparing AdaZeta/r=8/BS=16 with LoRA/r=1/BS=1, we still achieve nearly a 50\% reduction in memory usage. However, we would like to remark that the batch size of 1 setup is rarely used in practice due to the following reasons:
 \begin{itemize}
     \item First, reducing the batch size will dramatically increase the training time of the LoRA method.
     \item Second, such a small batch size leads to large stochastic noise during the fine-tuning process, which further harms the training performance. \cite{hu2023llm}
 \end{itemize}
\section{Conclusion}
In this paper, we propose an adaptive zeroth-order fine-tuning framework with tensor-train decomposition, named AdaZeta. Compared with previous ZO fine-tuning works, the AdaZeta method achieves significantly better fine-tuning results across various tasks and models. Theoretical analysis has confirmed that our proposed methods enjoy better convergence, which is consistent with our experimental results on both Roberta-Large and Llama-2 models across various fine-tuning tasks.\\~\\
Future work could explore improving the efficiency of the AdaZeta method by implementing distributed optimization across multiple GPUs for handling multiple queries concurrently at each step. Additionally, applying the adaptive query schedule to other PEFT methods may yield significantly better performance compared to the original MeZO algorithm.
\section*{Acknowledgements}
This project was supported by Amazon. We extend our gratitude to Siegfried Kunzmann, Jiajun Zhou, Clement Chung, Samridhi Choudhary, Hieu Nguyen and the many other colleagues at Amazon AGI and UCSB who engaged in discussions that shaped this work.\\~\\
This research also utilized resources from the National Energy Research Scientific Computing Center (NERSC), a U.S. Department of Energy Office of Science User Facility, supported under Contract No. DE-AC02-05CH11231 through NERSC award ASCR-ERCAP0030039.
\section*{Limitations}
The primary limitation of this work is related to accelerating the proposed method. Currently, multiple queries at each training step are executed sequentially in a for-loop, which restricts further speed enhancements. This process can potentially be optimized by implementing parallel or distributed optimization techniques on GPUs, allowing for the simultaneous execution of multiple queries, as these queries are independent of each other with different random seeds.
\section*{Potential Risks}
This paper provides a cost-effective solution that operates with a minimal memory footprint. Even though we need to fine-tune large-scale models, the proposed method can alleviate the burden on data centers and reduce $CO_{2}$ emissions. However, we acknowledge that prolonged training times, especially with multiple GPUs, can pose environmental challenges. Consequently, our ongoing research endeavors are focused on developing more efficient training methods and preserving computational power with ecological considerations in mind.
\bibliography{custom}

\begin{thebibliography}{41}
\providecommand{\natexlab}[1]{#1}

\bibitem[{Amari(1993)}]{amari1993backpropagation}
Shun-ichi Amari. 1993.
\newblock Backpropagation and stochastic gradient descent method.
\newblock \emph{Neurocomputing}, 5(4-5):185--196.

\bibitem[{Bowman et~al.(2015)Bowman, Angeli, Potts, and Manning}]{bowman2015large}
Samuel Bowman, Gabor Angeli, Christopher Potts, and Christopher~D Manning. 2015.
\newblock A large annotated corpus for learning natural language inference.
\newblock In \emph{Proceedings of the 2015 Conference on Empirical Methods in Natural Language Processing}, pages 632--642.

\bibitem[{Brown et~al.(2020)Brown, Mann, Ryder, Subbiah, Kaplan, Dhariwal, Neelakantan, Shyam, Sastry, Askell et~al.}]{brown2020language}
Tom Brown, Benjamin Mann, Nick Ryder, Melanie Subbiah, Jared~D Kaplan, Prafulla Dhariwal, Arvind Neelakantan, Pranav Shyam, Girish Sastry, Amanda Askell, et~al. 2020.
\newblock Language models are few-shot learners.
\newblock \emph{Advances in neural information processing systems}, 33:1877--1901.

\bibitem[{Bubeck et~al.(2015)}]{bubeck2015convex}
S{\'e}bastien Bubeck et~al. 2015.
\newblock Convex optimization: Algorithms and complexity.
\newblock \emph{Foundations and Trends{\textregistered} in Machine Learning}, 8(3-4):231--357.

\bibitem[{Chen et~al.(2019)Chen, Liu, Xu, Li, Lin, Hong, and Cox}]{chen2019zo}
Xiangyi Chen, Sijia Liu, Kaidi Xu, Xingguo Li, Xue Lin, Mingyi Hong, and David Cox. 2019.
\newblock Zo-adamm: Zeroth-order adaptive momentum method for black-box optimization.
\newblock \emph{Advances in neural information processing systems}, 32.

\bibitem[{Cheng et~al.(2023)Cheng, Zhu, Yao, Li, Li, and Zou}]{cheng2023ghostt5}
Xuxin Cheng, Zhihong Zhu, Ziyu Yao, Hongxiang Li, Yaowei Li, and Yuexian Zou. 2023.
\newblock Ghostt5: generate more features with cheap operations to improve textless spoken question answering.
\newblock In \emph{Proc. INTERSPEECH}, pages 1134--1138.

\bibitem[{Dettmers et~al.(2024)Dettmers, Pagnoni, Holtzman, and Zettlemoyer}]{dettmers2024qlora}
Tim Dettmers, Artidoro Pagnoni, Ari Holtzman, and Luke Zettlemoyer. 2024.
\newblock Qlora: Efficient finetuning of quantized llms.
\newblock \emph{Advances in Neural Information Processing Systems}, 36.

\bibitem[{Dua et~al.(2019)Dua, Wang, Dasigi, Stanovsky, Singh, and Gardner}]{dua2019drop}
Dheeru Dua, Yizhong Wang, Pradeep Dasigi, Gabriel Stanovsky, Sameer Singh, and Matt Gardner. 2019.
\newblock Drop: A reading comprehension benchmark requiring discrete reasoning over paragraphs.
\newblock \emph{arXiv preprint arXiv:1903.00161}.

\bibitem[{Gautam et~al.(2024)Gautam, Park, Zhou, Raman, and Ha}]{gautam2024variance}
Tanmay Gautam, Youngsuk Park, Hao Zhou, Parameswaran Raman, and Wooseok Ha. 2024.
\newblock Variance-reduced zeroth-order methods for fine-tuning language models.
\newblock \emph{arXiv preprint arXiv:2404.08080}.

\bibitem[{Ghadimi and Lan(2013)}]{ghadimi2013stochastic}
Saeed Ghadimi and Guanghui Lan. 2013.
\newblock Stochastic first-and zeroth-order methods for nonconvex stochastic programming.
\newblock \emph{SIAM journal on optimization}, 23(4):2341--2368.

\bibitem[{Houlsby et~al.(2019)Houlsby, Giurgiu, Jastrzebski, Morrone, De~Laroussilhe, Gesmundo, Attariyan, and Gelly}]{houlsby2019parameter}
Neil Houlsby, Andrei Giurgiu, Stanislaw Jastrzebski, Bruna Morrone, Quentin De~Laroussilhe, Andrea Gesmundo, Mona Attariyan, and Sylvain Gelly. 2019.
\newblock Parameter-efficient transfer learning for nlp.
\newblock In \emph{International Conference on Machine Learning}, pages 2790--2799. PMLR.

\bibitem[{Hu et~al.(2021)Hu, Shen, Wallis, Allen-Zhu, Li, Wang, Wang, and Chen}]{hu2021lora}
Edward~J Hu, Yelong Shen, Phillip Wallis, Zeyuan Allen-Zhu, Yuanzhi Li, Shean Wang, Lu~Wang, and Weizhu Chen. 2021.
\newblock Lora: Low-rank adaptation of large language models.
\newblock \emph{arXiv preprint arXiv:2106.09685}.

\bibitem[{Hu et~al.(2023)Hu, Lan, Wang, Xu, Lim, Lee, Bing, and Poria}]{hu2023llm}
Zhiqiang Hu, Yihuai Lan, Lei Wang, Wanyu Xu, Ee-Peng Lim, Roy Ka-Wei Lee, Lidong Bing, and Soujanya Poria. 2023.
\newblock Llm-adapters: An adapter family for parameter-efficient fine-tuning of large language models.
\newblock \emph{arXiv preprint arXiv:2304.01933}.

\bibitem[{Jiang et~al.(2024)Jiang, Chen, Pan, Xiang, Lin, Wu, Liu, and Song}]{jiang2024zo}
Shuoran Jiang, Qingcai Chen, Youcheng Pan, Yang Xiang, Yukang Lin, Xiangping Wu, Chuanyi Liu, and Xiaobao Song. 2024.
\newblock Zo-adamu optimizer: Adapting perturbation by the momentum and uncertainty in zeroth-order optimization.
\newblock In \emph{Proceedings of the AAAI Conference on Artificial Intelligence}, volume~38, pages 18363--18371.

\bibitem[{Kenton and Toutanova(2019)}]{kenton2019bert}
Jacob Devlin Ming-Wei~Chang Kenton and Lee~Kristina Toutanova. 2019.
\newblock Bert: Pre-training of deep bidirectional transformers for language understanding.
\newblock In \emph{Proceedings of NAACL-HLT}, pages 4171--4186.

\bibitem[{Kumar et~al.(2021)Kumar, Raghunathan, Jones, Ma, and Liang}]{kumar2021fine}
Ananya Kumar, Aditi Raghunathan, Robbie~Matthew Jones, Tengyu Ma, and Percy Liang. 2021.
\newblock Fine-tuning can distort pretrained features and underperform out-of-distribution.
\newblock In \emph{International Conference on Learning Representations}.

\bibitem[{Li and Liang(2021)}]{li2021prefix}
Xiang~Lisa Li and Percy Liang. 2021.
\newblock Prefix-tuning: Optimizing continuous prompts for generation.
\newblock In \emph{Proceedings of the 59th Annual Meeting of the Association for Computational Linguistics and the 11th International Joint Conference on Natural Language Processing (Volume 1: Long Papers)}, pages 4582--4597.

\bibitem[{Liu et~al.(2022)Liu, Tam, Muqeeth, Mohta, Huang, Bansal, and Raffel}]{liu2022few}
Haokun Liu, Derek Tam, Mohammed Muqeeth, Jay Mohta, Tenghao Huang, Mohit Bansal, and Colin~A Raffel. 2022.
\newblock Few-shot parameter-efficient fine-tuning is better and cheaper than in-context learning.
\newblock \emph{Advances in Neural Information Processing Systems}, 35:1950--1965.

\bibitem[{Liu et~al.(2018)Liu, Kailkhura, Chen, Ting, Chang, and Amini}]{liu2018zeroth}
Sijia Liu, Bhavya Kailkhura, Pin-Yu Chen, Paishun Ting, Shiyu Chang, and Lisa Amini. 2018.
\newblock Zeroth-order stochastic variance reduction for nonconvex optimization.
\newblock \emph{Advances in Neural Information Processing Systems}, 31.

\bibitem[{Liu et~al.(2019)Liu, Ott, Goyal, Du, Joshi, Chen, Levy, Lewis, Zettlemoyer, and Stoyanov}]{liu2019roberta}
Yinhan Liu, Myle Ott, Naman Goyal, Jingfei Du, Mandar Joshi, Danqi Chen, Omer Levy, Mike Lewis, Luke Zettlemoyer, and Veselin Stoyanov. 2019.
\newblock Roberta: A robustly optimized bert pretraining approach.
\newblock \emph{arXiv preprint arXiv:1907.11692}.

\bibitem[{Liu et~al.(2024)Liu, Zhu, Gong, Cheng, Hsieh, and You}]{liu2024sparse}
Yong Liu, Zirui Zhu, Chaoyu Gong, Minhao Cheng, Cho-Jui Hsieh, and Yang You. 2024.
\newblock Sparse mezo: Less parameters for better performance in zeroth-order llm fine-tuning.
\newblock \emph{arXiv preprint arXiv:2402.15751}.

\bibitem[{Lohr(2009)}]{lohr2021sampling}
Sharon~L Lohr. 2009.
\newblock \emph{Sampling: design and analysis}.
\newblock Nelson Education.

\bibitem[{Loshchilov and Hutter(2018)}]{loshchilov2018decoupled}
Ilya Loshchilov and Frank Hutter. 2018.
\newblock Decoupled weight decay regularization.
\newblock In \emph{International Conference on Learning Representations}.

\bibitem[{Malladi et~al.(2023)Malladi, Gao, Nichani, Damian, Lee, Chen, and Arora}]{malladi2023fine}
Sadhika Malladi, Tianyu Gao, Eshaan Nichani, Alex Damian, Jason~D Lee, Danqi Chen, and Sanjeev Arora. 2023.
\newblock Fine-tuning language models with just forward passes.
\newblock \emph{arXiv preprint arXiv:2305.17333}.

\bibitem[{Nesterov and Spokoiny(2017)}]{nesterov2017random}
Yurii Nesterov and Vladimir Spokoiny. 2017.
\newblock Random gradient-free minimization of convex functions.
\newblock \emph{Foundations of Computational Mathematics}, 17(2):527--566.

\bibitem[{Novikov et~al.(2015)Novikov, Podoprikhin, Osokin, and Vetrov}]{novikov2015tensorizing}
Alexander Novikov, Dmitrii Podoprikhin, Anton Osokin, and Dmitry~P Vetrov. 2015.
\newblock Tensorizing neural networks.
\newblock \emph{Advances in neural information processing systems}, 28.

\bibitem[{Oseledets(2011)}]{oseledets2011tensor}
Ivan~V Oseledets. 2011.
\newblock Tensor-train decomposition.
\newblock \emph{SIAM Journal on Scientific Computing}, 33(5):2295--2317.

\bibitem[{Pang et~al.(2002)Pang, Lee, and Vaithyanathan}]{pang2002thumbs}
Bo~Pang, Lillian Lee, and Shivakumar Vaithyanathan. 2002.
\newblock Thumbs up? sentiment classification using machine learning techniques.
\newblock In \emph{Proceedings of the ACL-02 conference on Empirical methods in natural language processing-Volume 10}, pages 79--86.

\bibitem[{Rajpurkar et~al.(2016)Rajpurkar, Zhang, Lopyrev, and Liang}]{rajpurkar2016squad}
Pranav Rajpurkar, Jian Zhang, Konstantin Lopyrev, and Percy Liang. 2016.
\newblock Squad: 100,000+ questions for machine comprehension of text.
\newblock \emph{arXiv preprint arXiv:1606.05250}.

\bibitem[{Socher et~al.(2013)Socher, Perelygin, Wu, Chuang, Manning, Ng, and Potts}]{socher2013recursive}
Richard Socher, Alex Perelygin, Jean Wu, Jason Chuang, Christopher~D Manning, Andrew~Y Ng, and Christopher Potts. 2013.
\newblock Recursive deep models for semantic compositionality over a sentiment treebank.
\newblock In \emph{Proceedings of the 2013 conference on empirical methods in natural language processing}, pages 1631--1642.

\bibitem[{Tian et~al.(2023)Tian, Fang, Wang, and Wang}]{tian2023bebert}
Jiayi Tian, Chao Fang, Haonan Wang, and Zhongfeng Wang. 2023.
\newblock Bebert: Efficient and robust binary ensemble bert.
\newblock In \emph{ICASSP 2023-2023 IEEE International Conference on Acoustics, Speech and Signal Processing (ICASSP)}, pages 1--5. IEEE.

\bibitem[{Touvron et~al.(2023)Touvron, Lavril, Izacard, Martinet, Lachaux, Lacroix, Rozi{\`e}re, Goyal, Hambro, Azhar et~al.}]{touvron2023LLaMA}
Hugo Touvron, Thibaut Lavril, Gautier Izacard, Xavier Martinet, Marie-Anne Lachaux, Timoth{\'e}e Lacroix, Baptiste Rozi{\`e}re, Naman Goyal, Eric Hambro, Faisal Azhar, et~al. 2023.
\newblock Llama: Open and efficient foundation language models.
\newblock \emph{arXiv preprint arXiv:2302.13971}.

\bibitem[{Wang et~al.(2019)Wang, Pruksachatkun, Nangia, Singh, Michael, Hill, Levy, and Bowman}]{wang2019superglue}
Alex Wang, Yada Pruksachatkun, Nikita Nangia, Amanpreet Singh, Julian Michael, Felix Hill, Omer Levy, and Samuel Bowman. 2019.
\newblock Superglue: A stickier benchmark for general-purpose language understanding systems.
\newblock \emph{Advances in neural information processing systems}, 32.

\bibitem[{Wang et~al.(2018)Wang, Singh, Michael, Hill, Levy, and Bowman}]{wang2018glue}
Alex Wang, Amanpreet Singh, Julian Michael, Felix Hill, Omer Levy, and Samuel~R Bowman. 2018.
\newblock Glue: A multi-task benchmark and analysis platform for natural language understanding.
\newblock \emph{arXiv preprint arXiv:1804.07461}.

\bibitem[{Williams et~al.(2017)Williams, Nangia, and Bowman}]{williams2017broad}
Adina Williams, Nikita Nangia, and Samuel~R Bowman. 2017.
\newblock A broad-coverage challenge corpus for sentence understanding through inference.
\newblock \emph{arXiv preprint arXiv:1704.05426}.

\bibitem[{Williams et~al.(2018)Williams, Nangia, and Bowman}]{williams2018broad}
Adina Williams, Nikita Nangia, and Samuel~R Bowman. 2018.
\newblock A broad-coverage challenge corpus for sentence understanding through inference.
\newblock In \emph{Proceedings of NAACL-HLT}, pages 1112--1122.

\bibitem[{Xu et~al.()Xu, Ye, Li, and Chen}]{xu2024can}
Han Xu, Jingyang Ye, Yutong Li, and Haipeng Chen.
\newblock Can speculative sampling accelerate react without compromising reasoning quality?
\newblock In \emph{The Second Tiny Papers Track at ICLR 2024}.

\bibitem[{Yang et~al.(2024{\natexlab{a}})Yang, Zhou, Wong, and Zhang}]{yang2024loretta}
Yifan Yang, Jiajun Zhou, Ngai Wong, and Zheng Zhang. 2024{\natexlab{a}}.
\newblock Loretta: Low-rank economic tensor-train adaptation for ultra-low-parameter fine-tuning of large language models.
\newblock \emph{arXiv preprint arXiv:2402.11417}.

\bibitem[{Yang et~al.(2024{\natexlab{b}})Yang, Choudhary, Xie, Gao, Kunzmann, and Zhang}]{yang2024comera}
Zi~Yang, Samridhi Choudhary, Xinfeng Xie, Cao Gao, Siegfried Kunzmann, and Zheng Zhang. 2024{\natexlab{b}}.
\newblock Comera: Computing-and memory-efficient training via rank-adaptive tensor optimization.
\newblock \emph{arXiv preprint arXiv:2405.14377}.

\bibitem[{Zaken et~al.(2022)Zaken, Goldberg, and Ravfogel}]{zaken2022bitfit}
Elad~Ben Zaken, Yoav Goldberg, and Shauli Ravfogel. 2022.
\newblock Bitfit: Simple parameter-efficient fine-tuning for transformer-based masked language-models.
\newblock In \emph{Proceedings of the 60th Annual Meeting of the Association for Computational Linguistics (Volume 2: Short Papers)}, pages 1--9.

\bibitem[{Zhang et~al.(2024)Zhang, Ladhak, Durmus, Liang, McKeown, and Hashimoto}]{zhang2024benchmarking}
Tianyi Zhang, Faisal Ladhak, Esin Durmus, Percy Liang, Kathleen McKeown, and Tatsunori~B Hashimoto. 2024.
\newblock Benchmarking large language models for news summarization.
\newblock \emph{Transactions of the Association for Computational Linguistics}, 12:39--57.

\end{thebibliography}

\clearpage
\appendix
\section{Detail of Experiment Setup}\label{app:exp}
\subsection{Dataset Setup}
\begin{table}[h]
\centering
\caption{Metrics that we use to evaluate the benchmark for the Roberta-Large Model.}
\label{tab:glue_metric}
\resizebox{0.25\textwidth}{!}{%
\begin{tabular}{@{}cc@{}}
\toprule
Task Name & Metric                       \\ \midrule
SST-2     & Accuracy                      \\
SST-5   & Accuracy                      \\
QNLI      & Accuracy                          \\
MNLI      & Matched Acc.          \\
SNLI      &  Accuracy                   \\
RTE       & Accuracy                          \\        \bottomrule
\end{tabular}
}
\end{table}
Our research utilized a variety of tasks to measure the performance of the Roberta-Large model, including sentiment analysis (SST-2, SST-5 \cite{socher2013recursive}, MR \cite{pang2002thumbs}), and natural language inference (MNLI \cite{wang2018glue}, QNLI \cite{williams2018broad}, SNLI \cite{bowman2015large}, RTE \cite{wang2018glue}) tasks. Table \ref{tab:glue_metric} summarizes the evaluation metrics used for these tasks.\\~\\
Further, we extended our experiments on a large-scale Llama-2-7B model to include tasks from the SuperGLUE benchmark \cite{wang2019superglue}, which involves both classification (CB, BoolQ, WSC) and reasoning tasks (COPA and ReCoRD), as well as additional generation tasks, SQuAD \cite{rajpurkar2016squad}. For these tests, we introduced a challenging low-resource data condition, limiting our samples to 1,000 for training, 500 for validation, and 1,000 for testing, as detailed in the prompt-based task settings from Appendix D of \cite{malladi2023fine}. The metrics for these evaluations are outlined in Table \ref{tab:superglue_metric}.
\begin{table}[h]
\centering
\caption{Metrics that we use to evaluate SuperGLUE and generations tasks.}
\label{tab:superglue_metric}
\resizebox{0.2\textwidth}{!}{%
\begin{tabular}{@{}cc@{}}
\toprule
Task Name & Metric                       \\ \midrule
CB      & F1                          \\
BoolQ     & Accuracy                      \\
WSC      & F1          \\
COPA      & Accuracy                        \\
ReCoRD      & F1                                    \\
SQuAD    & F1                       \\
\bottomrule
\end{tabular}%
}
\vspace{-20pt}
\end{table}
\subsection{Baselines}
In this section, we provide a detailed introduction to the baseline method considered in our experiments, which are listed as follows: \\~\\
\textbf{Full-model First-Order Fine-Tuning (FT)} is the most widely used method for fine-tuning LLMs. In this process, the model is initialized with pre-trained weights, and all model parameters are updated by the first-order optimizer. In this paper, the AdamW optimizer \cite{loshchilov2018decoupled} is used to conduct the first-order experiments.\\~\\
\textbf{Zero-shot/In-context-learning (ICL)} is the most widely used method for fine-tuning large language models (LLMs). In this process, the model is initialized with pre-trained weights, and all model parameters are updated by the first-order (FO) optimizer. In this paper, the AdamW optimizer \cite{loshchilov2018decoupled} is used to conduct the first-order experiments.\\~\\
\textbf{Linear-probing (LP)} method involves freezing the pretrained weights of the model and adding a final linear classifier layer, implemented using the scipy package. By fine-tuning this layer with the first-order method, we only need to construct a small backpropagation graph. However, this method is not suitable for generative tasks. Therefore, we only apply the LP method in the Roberta-Large experiments.\\~\\
\textbf{Memory-Efficient Zeroth-Order (MeZO)} was first proposed in \cite{malladi2023fine}, which fine-tunes LLMs using only the forward pass. The MeZO method significantly reduces memory costs by eliminating the need for a backpropagation graph and has demonstrated superior performance compared to inference-only methods like Zero-shot, ICT, and LP methods across various downstream tasks.\\~\\
\textbf{Memory-Efficient Zeroth-Order with LoRA adapters (MeZO-LoRA)} is a derivative method introduced in \cite{malladi2023fine}, which freezes the pretrained weights and fine-tunes only the injected LoRA adapters \cite{hu2021lora}. The MeZO-LoRA method is the most relevant baseline in this field compared to our work. However, its performance improvement over the MeZO method is limited, and the mechanisms behind zeroth-order parameter-efficient fine-tuning are not extensively discussed.\\~\\
\textbf{Sparse Memory-efficient Zeroth-Order (Sparse-MeZO)}  is a recently proposed method aiming to enhance the performance and convergence speed of the MeZO method \cite{liu2024sparse}. However, as the code and detailed layer-wise hyperparameter setup have not been released, we have reproduced the method using a fixed sparsity ratio for each layer. This ratio is selected based on the best overall outcome as presented in Fig. 6 of their paper.
\subsection{Hyperparameters}
In this section, we outline the detailed setup of hyperparameters utilized in our study. The specific choices of hyperparameters, such as learning rate, training steps, and batch size, are summarized in Table \ref{tab:setup_1}. In our experiments, we strive to maintain a consistent learning rate across different methods for the same tasks. However, for approaches like full-model fine-tuning, we opt for a lower learning rate to ensure convergence. This principle is also applied in our large-scale experiments on the Llama-2-7B model, details of which are summarized in Table \ref{tab:setup_2}.\\~\\
In addition to the standard hyperparameter configuration, we also consider the shape of tensor factors in our methods. To represent a layer with input and output dimensions of \(o\) and \(p\), respectively, we employ a list of \(m\) tensor factors \(\mathcal{G}_i \in \mathbb{R}^{r \times k_i r}\), where the product \(\Pi k_1 \cdots k_m = o \cdot p\). The specific shapes of \(k_i\) corresponding to different values of \(o\) and \(p\), given a bottleneck size of 8 or 64 for the tensorized methods, are detailed in Table \ref{tab:tensor_shape}. Note that the optimal factors shape and tensor rank for the tensor-train method can only be determined by the experiments' trail. However, previous work also explores the possibility of utilizing the adaptive rank to improve the performance \cite{yang2024comera}, which may further improve the performance of our AdaZeta method.
\begin{table}[t]
\centering
\caption{The hyperparameter grids used for Roberta-Large experiments are detailed as follows. We fine-tune each task for 80K steps, except for the FT method, which is conducted over 20 epochs. We record the best model checkpoint based on the validation loss every 200 training steps.}
\label{tab:setup_1}
\resizebox{0.49\textwidth}{!}{%
\begin{tabular}{@{}ccc@{}}
\toprule
Experiment & Hyperparameters & Values \\ \midrule
FT         & Batch size      &    \{8, 16, 64\}   \\
           & Learning rate   &   \{1e-6, 5e-7\}       \\ \midrule
MeZO         & Batch size      &    \{16, 64\}   \\
           & Learning rate   &   \{1e-6, 5e-7\}        \\ 
           & $\epsilon$  &   1e-3       \\ \midrule
MeZO-LoRA         & Batch size      &    \{16, 64\}   \\
           & Learning rate   &   \{1e-4, 5e-5\}        \\
           & LoRA rank & 8\\
           & $\epsilon$  &   1e-3       \\ \midrule
Sparse-MeZO         & Batch size      &    \{16, 64\}   \\
           & Learning rate   &   \{1e-5, 1e-6\}       \\ 
           & sparse ratio & 0.75\\
           & $\epsilon$  &   1e-3       \\ \midrule

AdaZeta      & Batch size      &   \{16, 64\}     \\
           & Learning rate   &    \{1e-4, 5e-5\}      \\
           & Bottleneck dimension   &    64    \\
            & Tensor Rank    &    5  \\
            & $\epsilon$  &   1e-3        \\
            \midrule
\end{tabular}%
}
\end{table}

\begin{table}[t]
\centering
\caption{The hyperparameter grids used for Llama-2-7B experiments are outlined as follows. We fine-tune each task for 5K steps using our AdaZeta method, 10K steps for other ZO fine-tuning methods (MeZO, MeZO-LoRA, Sparse-MeZO), and 5 epochs for the first-order Full-model Fine-Tuning (FT) method. We record the best model checkpoint based on the validation loss every 200 training steps.}
\label{tab:setup_2}
\resizebox{0.5\textwidth}{!}{%
\begin{tabular}{@{}ccc@{}}
\toprule
Experiment & Hyperparameters & Values \\ \midrule
FT         & Batch size      &    \{8, 16, 64\}   \\
           & Learning rate   &   \{1e-6, 5e-7\}       \\ \midrule
MeZO         & Batch size      &    \{16, 64\}   \\
           & Learning rate   &   \{1e-6, 5e-7\}        \\ 
           & $\epsilon$  &   1e-3       \\ \midrule
MeZO-LoRA         & Batch size      &    \{16, 64\}   \\
           & Learning rate   &   \{1e-4, 5e-5\}        \\
           & LoRA rank & \{5, 8, 16\}\\
           & $\epsilon$  &   1e-3      \\ \midrule
Sparse-MeZO         & Batch size      &    \{16, 64\}   \\
           & Learning rate   &   \{1e-5, 1e-6\}       \\ 
           & sparse ratio & 0.75\\
           & $\epsilon$  &   1e-3       \\ \midrule

AdaZeta      & Batch size      &   \{16, 64\}     \\
           & Learning rate   &    \{1e-4, 5e-5\}      \\
           & Bottleneck dimension   &    \{8, 64\}    \\
            & Tensor Rank    &    \{5, 8, 16\}  \\
            & Query Constants & $\alpha=0.85, \beta=0.45$ \\
            & Maximum Query & $Q_{max}=20$ \\
            & $\epsilon$  &   1e-3        \\
            \midrule
\end{tabular}%
}
\end{table}

\begin{table}[!h]
\centering
\caption{The shape settings of the tensorized adapters in AdaZeta Method}
\label{tab:tensor_shape}
\resizebox{0.48\textwidth}{!}{%
\begin{tabular}{@{}ccc@{}}
\toprule
Bottleneck size & Matrix Shape & Tensor Shape\\ 
\midrule
    8 &  $768\times 64$      &   [8, 8, 12, 4, 4, 4]    \\
                        &      $4096\times 64$   &   [16, 16, 16, 4, 4, 4]       \\
                       &       $64\times 768$      &   [4, 4, 4, 12, 8, 8]    \\
                        &      $64\times 4096$   &   [4, 4, 4, 16, 16, 16]       \\
                             \midrule
                             
    64 &  $768\times 8$      &   [8, 8, 12, 2, 2, 2]    \\ 
                        &      $4096\times 8$   &   [16, 16, 16, 2, 2, 4]       \\
                       &       $8\times 768$      &   [2, 2, 2, 12, 8, 8]    \\
                        &      $8\times 4096$   &   [2, 2, 2, 16, 16, 16]       \\                  
                           \bottomrule
\end{tabular}%
}
\end{table}

\section{Additional Experiments}
\subsection{Additional Momeory Comparison results}\label{app:memory}
In this section, we provide more quantitative results about the training memory comparison between the FO and ZO fine-tuning methods. In addition to the training memory on SST2 tasks we measure in Section \ref{sec:mem}, we further profile the memory cost on WIC, CB, and MultiRC tasks. The results are shown in Table \ref{tab:memory_quant}.\\~\\
We can observe from the table that the AdaZeta method achieves 5-8$\times$ memory reduction on different tasks. Also, the AdaZeta method utilizes similar or even less memory than the other MeZO, MeZo-LoRA, and Sparse-MeZO methods with an additional variance reduction feature, which largely improves the ZO fine-tuning accuracy.
\begin{table}[!h]
\centering
\caption{Quantitative results for the memory profiling over SST2 and MultiRC tasks.}
\label{tab:memory_quant}
\resizebox{0.5\textwidth}{!}{%
\begin{tabular}{c|cccc}
\hline
Methods          & SST2                      & WIC                      & CB                   & \multicolumn{1}{l}{MultiRC} \\ \hline
FT               & 118.65                    & 115.3                    & 151.97                 & 191.97                      \\
MeZO             & 15.08                     & 15.22                    & 23.01                & 41.17                       \\
MeZO-LoRA        & 14.75                     & 15.23                    & 23.02                & 41.18                       \\
MeZO-LoRA(BS=64) & 21.07 & 25.30 & 71.70 & 84.30        \\
Sparse-MeZO      & 14.35                     & 15.21                    & 23.01                & 42.13                       \\
AdaZeta          & 14.73                     & 15.22                    & 23.01                & 41.17                       \\ \hline
\end{tabular}
}
\end{table}

\onecolumn
\section{Proof of Theorem \ref{the:convergence}}\label{app:proof}
To retain the readability of the proof, we use a single-column format in the following. To provide the proof of Theorem \ref{the:convergence}, we first present a Lemma regarding the bound of gradient noise. Recall from the gradient estimation rule that:
\begin{align}
    \nabla\hat{\ell}(\bm{w}_k) &= \frac{1}{B} \sum_{b_i \in \mathcal{B}} \hat{g}(\bm{w}_k; b_i) \\
    \hat{g}(\bm{w}_k; b_i) &= \frac{1}{Q_k} \sum_{j=1}^{Q_k} \hat{g}(\bm{w}_k; b_i, u_{i,j}),
\end{align}
where there are two sources of randomness:
a) The randomness leads by the mini-batch sampling and b) The randomness leads by the presence of ZO gradient estimation. Based on these two randomnesses, we define two gradient noises as $h_k$ and $e_k$, respectively.
\begin{align}\label{eq:noise}
    h_k &:= \nabla\hat{\ell}(\bm{w}_k) - \nabla \ell(\bm{w}_k) = \frac{1}{B} \sum_{b_i \in \mathcal{B}} \hat{g}(\bm{w}_k; b_i) - \nabla\ell(\bm{w}_k) \\
    e_k &:= \hat{g}(\bm{w}_k; b_i) - \nabla\ell(\bm{w}_k) = \frac{1}{Q_k} \sum_{j=1}^{Q_k} \hat{g}(\bm{w}_k; b_i, u_{i,j}) - \nabla\ell(\bm{w}_k)
\end{align}
Here, we first bound the gradient noise $h_k$ with a fact given in stochastic gradient descent theory. We consider the noise concerning the mean of the ZO estimated gradient \(\nabla \ell(\bm{w}_k)\), where the loss function \(\ell\) is a randomized smoothing version of \(\ell\).

\begin{lemma}\label{lemma:noise}
   Based on the definition in eq. (\ref{eq:noise}) and the Assumption A2, we can bound the L2-norm of the gradient noise $h_k$ by taking expectation:
\begin{align}
   \mathbb{E}[\|h_k\|^2] \leq \frac{N - B}{NB(B-1)Q_k}\sum_i(2d\delta^2+\frac{\epsilon^2L^2d^2}{2} + 2\delta^2)
\end{align}
\end{lemma}
\begin{proof}For convenience, we consider a general case that the mini-batch $\mathcal{B}$ is formed by uniform sampling without replacement and follows the i.i.d. fashion. Then, according to \cite{lohr2021sampling}[Section 2.8, Page 48], the following holds for a random sampling noise:
\begin{align}\label{eq:3}
    \mathbb{E}[\|h_k\|^2] = \frac{N - B}{NB}\Lambda^2, 
\end{align}
where $\Lambda^2$ is the sample variance of the gradient $\hat{g}(\bm{w}_k;b_i)$, which is defined as:
\begin{align}\label{eq:lambda}
    \Lambda^2 & = \frac{1}{B-1}\sum_{i=1}^B\|\hat{g}(\bm{w}_k;b_i) - \nabla\ell(\bm{w}_k)\|^2\\
    & = \frac{1}{B-1}\sum_{i=1}^B\|\nabla\ell(\bm{w}_k) + e_k - \nabla\ell(\bm{w}_k)\|^2 \\
    & =\frac{1}{B-1}\sum_{i=1}^B \|e_k\|^2 ,
\end{align}
where $e_k$ is defined as the gradient noise leads by the ZO estimation in eq. (\ref{eq:noise}).\\~\\
Finally, we need to bound the variance $\Lambda^2$, related to the ZO gradient estimation noise. Taking expectation with respect to the i.i.d. random perturbation vector $\bm{u}$, we have:
\begin{align}\label{eq:lambda}
    \mathbb{E}_{\bm{u}}[\Lambda^2]
    & \leq \mathbb{E}_{\bm{u}}[\frac{1}{B-1}\sum_{i=1}^B\|e_k^i \|^2]\\
    & \leq \frac{1}{(B-1)Q_k^2}\sum_i \mathbb{E}_{\bm{u}}[\|\sum_{j=1}^{Q_k}( \hat{g}(\bm{w}_k; b_i, u_{i,j}) - \nabla\ell(\bm{w}_k))\|] \\
    & \overset{(a)}{=}\frac{1}{(B-1)Q_k}\sum_i\mathbb{E}_{\bm{u}}[\|\hat{g}(\bm{w}_k; b_i, u_{i,1}) -\nabla\ell(\bm{w}_k) \|],
\end{align}
where (a) is given under the case that $u_{i,j}$ is i.i.d, which obtain:
\begin{align}
    \mathbb{E}_{\bm{u}}[\|\hat{g}(\bm{w}_k; b_i, u_{i,j}) -\nabla\ell(\bm{w}_k) \|] = \mathbb{E}_{\bm{u}}[\|\hat{g}(\bm{w}_k; b_i, u_{i,1}) -\nabla\ell(\bm{w}_k) \|]
\end{align}
Finally, we need to bound the term $\mathbb{E}_{\bm{u}}[\|\hat{g}(\bm{w}_k; b_i, u_{i,1}) -\nabla\ell(\bm{w}_k) \|]$, which gives:
\begin{align}\label{eq:final}
    &\mathbb{E}_{\bm{u}}[\|\hat{g}(\bm{w}_k; b_i, u_{i,1}) -\nabla\ell(\bm{w}_k) \|]] \\
    &\leq \mathbb{E}_{\bm{u}}[\|\hat{g}(\bm{w}_k; b_i, u_{i,1}) -\nabla\ell(\bm{w}_k; b_i) \|] + \mathbb{E}_{\bm{u}}[\|\nabla\ell(\bm{w}_k; b_i) - \nabla\ell(\bm{w}_k)\|] \\
    & \overset{(a)}{\leq} 2d\|\hat{g}(\bm{w}_k; b_i, u_{i,1})\|  +\frac{\epsilon^2L^2d^2}{2} +  \mathbb{E}_{\bm{u}}[\|\nabla\ell(\bm{w}_k)\|] +  \mathbb{E}_{\bm{u}}[\|\nabla\ell(\bm{w}_k; b_i)\|] \\
    &\overset{(b)}{\leq} 2d\delta^2+\frac{\epsilon^2L^2d^2}{2} + 2\delta^2,
\end{align}
where (a) follows a similar idea of the proof in \cite{ghadimi2013stochastic}[eq. (3.21)] and (b) is given by using the bound of the gradient in Assumption A2.\\~\\
Putting it all together we can obtain the upper bound for the gradient noise $\|h_k\|$.
\end{proof}
\noindent Now we begin to present the proof of Theorem \ref{the:convergence}:\\~\\
We start from the gradient updating rule in the AdaZeta algorithm, which gives $\bm{w}_{t+1} = \bm{w}_t - \eta \nabla \hat{\ell}(\bm{w}_k)$. By using Taylor's theorem on the exact smoothed loss $\ell(\bm{w}_k)$, we have:
\begin{align}
    \ell(\bm{w}_{k+1}) &= \ell(\bm{w}_k - \eta\nabla\hat{\ell}(\bm{w}_k))\\
    & = \ell(\bm{w}_k) - \eta \nabla\hat{\ell}_k(\bm{w}_k)^\top\nabla\ell(\bm{w}_k) + \frac{\eta^2}{2}\nabla\hat{\ell}(\bm{w}_k)^\top\nabla\ell(\bm{w}_k)^2\nabla\hat{\ell}(\bm{w}_k)
\end{align}
Taking expectations on both sides gives:
\begin{align*}
    \mathbb{E}_{\bm{w}_t}[\ell(\bm{w}_{k+1})] & = \mathbb{E}_{\bm{w}_t}[\ell(\bm{w}_k)] - \eta \mathbb{E}_{\bm{w}_t}[\nabla\hat{\ell}(\bm{w}_k)^\top\nabla\ell(\bm{w}_k)] + \frac{\eta^2}{2}\mathbb{E}_{\bm{w}_t}[\nabla\hat{\ell}(\bm{w}_k)^\top\nabla\ell(\bm{w}_k)^2\nabla\hat{\ell}(\bm{w}_k)]\\
    & \overset{\text{(a)}}{\leq}  \mathbb{E}_{\bm{w}_t}[\ell(\bm{w}_k)] - \eta \mathbb{E}_{\bm{w}_t}[\nabla\ell(\bm{w}_k)^2] + \frac{\eta^2L}{2}\mathbb{E}_{\bm{w}_t}[\nabla\hat{\ell}(\bm{w}_k)^2],
\end{align*}
where (a) can be proved with the use of the Lipschitz smoothness gradient denied in Assumption A1 that gives $x$ and $y$, we have $\|\nabla \ell(x) - \nabla \ell(y)\| \leq L\|x-y\|$. Additionally, by the mean value theorem for vector-valued functions, there exists for any point $c$ on the line segment between $x$ and $y$ such that:
\begin{align}
    \nabla f(y)-\nabla f(x)=\nabla^2 f(c)(y-x).
\end{align} Taking the norms on both sides and using the Lipschitz condition, we have:
\begin{align}
    \left\|\nabla^2 f(c)(y-x)\right\|=\|\nabla f(y)-\nabla f(x)\| \leq L\|y-x\|.
\end{align}
Finally, since this must hold for any $y$ and $x$, and since the norm of the Hessian matrix is the supremum of $\left\|V^2 f(c)(y-x)\right\| /\|y-x\|$ for non-zero $y-x$, it follows that:
\begin{align}
    \left\|\nabla^2 f(c)\right\| \leq L
\end{align}
Rearrange and we obtain:
\begin{align}
    \eta\mathbb{E}[\|\nabla\ell(\bm{w}_k)\|^2] \leq \mathbb{E}[\ell(\bm{w}_k)] - \mathbb{E}[\ell(\bm{w}_{k+1})] + \frac{\eta^2L}{2}\mathbb{E}[\nabla\hat{\ell}(\bm{w}_k)^2]
\end{align}
Taking summation over steps $k=1, \cdots, K$ gives:
\begin{align}\label{eq:1}
\sum_{k=1}^K    \eta\mathbb{E}[\|\nabla\ell(\bm{w}_k)\|^2] & \leq \mathbb{E}[\ell(\bm{w}_0) - \ell(\bm{w}_K)] + \sum_{k=1}^K\frac{\eta^2L}{2}\mathbb{E}[\nabla\hat{\ell}(\bm{w}_k)^2] \\ 
& \overset{\text{(a)}}{\leq}  \mathbb{E}[\ell_0 - \ell^*] + \epsilon^2 L + \sum_{k=1}^K\frac{\eta^2L}{2}\mathbb{E}[\nabla\hat{\ell}(\bm{w}_k)^2] \\
& \overset{\text{(b)}}{\leq} R + \epsilon^2 L  + \sum_{k=1}^K\frac{\eta^2L}{2}\mathbb{E}[\nabla\hat{\ell}(\bm{w}_k)^2],
\end{align}
where (a) is using the Lemma 1 in \cite{liu2018zeroth} that $\ell(\bm{w}_0) - \ell(\bm{w}_T)\leq \ell(\bm{w}_0) - \ell(\bm{w}_0) + \ell^* - \ell^*\leq (\ell(\bm{w}_0) - \ell^*) + \epsilon^2 L$ and (b) is given by setting $R:=\ell(\bm{w}_1) - \ell^*$. Now, the key to the bound comes from the last term in the right of the inequation.\\~\\
To bound the last term, we first represent the noise gradient $\nabla\hat{\ell}_k(\bm{w}_k)^2$ as a combination of the true gradient and the gradient noise introduced in eq. (\ref{eq:noise}), which gives:
\begin{align}\label{eq:2}
\nabla\hat{\ell}(\bm{w}_k):=\nabla\ell(\bm{w}_k) + h_k
\end{align}

Taking eq. (\ref{eq:2}) back into eq. (\ref{eq:1}), using the results from Lemma \ref{lemma:noise}, taking the expectation over all randomness and average over the maximum steps $K$, we obtain:
\begin{align*}
    &\frac{1}{K}\sum_{k=1}^K    \eta\mathbb{E}[\|\nabla\ell(\bm{w}_k)\|^2] \\
    &
    \leq \frac{R}{K} + \frac{\epsilon^2 L}{K}  + \frac{1}{K}\sum_{k=1}^K\frac{\eta^2L}{2}\mathbb{E}[\nabla\hat{\ell}(\bm{w}_k)^2]\\
    & \leq  \frac{R}{K} + \frac{\epsilon^2 L}{K}  + \frac{1}{K}\sum_{k=1}^K\frac{\eta^2L}{2}\mathbb{E}[\|\nabla\ell(\bm{w}_k)\| + \| h_k\|] \\
    &  = \frac{R}{K} + \frac{\epsilon^2 L}{K}  + \frac{\eta^2L\delta}{2}  + \frac{1}{K}\sum_{k=1}^K\frac{\eta^2L}{2} \frac{N - B}{NB}\Lambda^2 \\
    & \leq \frac{R}{K} + \frac{\epsilon^2 L}{K}  + \frac{\eta^2L\delta}{2}  + \frac{1}{K}\sum_{k=1}^K\frac{\eta^2L}{2} \frac{N - B}{NB} (\frac{1}{(B-1)Q_k}\sum_i\mathbb{E}_{\bm{u}}[\|e^{i,1}\|]]+ \frac{B\epsilon^2L}{2(B-1)})
    \end{align*}
    \begin{align*}
    & \leq \frac{R}{K} + \frac{\epsilon^2 L}{K}  + \frac{\eta^2L\delta}{2} + \frac{1}{K}\sum_{k=1}^K\frac{\eta^2L}{2} \frac{N - B}{NB} (\frac{\sum_i( 2d\delta^2+\frac{\epsilon^2L^2d^2}{2} + 2\delta)}{(B-1)Q_k}+ \frac{B\epsilon^2L}{2(B-1)})\\
    & = \frac{R + \epsilon^2 L +C(d,\epsilon) \sum_{k}\frac{1}{Q_k}}{K} + \frac{\eta^2L\delta}{2} + \frac{B\epsilon^2L}{2(B-1)} \\
    & = \mathcal{O}(\frac{R + \epsilon^2 L +C(d,\epsilon) \sum_{k}\frac{1}{Q_k}}{K}),
\end{align*}
where $C(d,\epsilon)$ is a constant defined as $C(d,\epsilon):=\sum_{k=1}^K\frac{\eta^2L}{2} \frac{N - B}{NB} (\frac{\sum_i( 2d\delta^2+\frac{\epsilon^2L^2d^2}{2} + 2\delta)}{(B-1)})$

Divide both side with $\eta$ and use the trick to introduce some randomly chosen $\bm{w}_T$ from the history with probability $P(T=k)=\frac{1}{K}$, we finish the proof as:
\begin{align}
    \mathbb{E}[\|\nabla\ell(\bm{w}_T)\|^2] = \frac{1}{K}\sum_{k=1}^K \mathbb{E}[\|\nabla\ell(\bm{w}_k)\|^2] \leq\mathcal{O}(\frac{R + \epsilon^2 L +C(d,\epsilon) \sum_{k}\frac{1}{Q_k}}{K\epsilon})
\end{align}

\end{document}